\documentclass{article}

\PassOptionsToPackage{numbers, compress}{natbib}

\usepackage[preprint]{neurips_2020}

\usepackage{amsmath,amsfonts,bm}

\def\eqref#1{equation~\ref{#1}}

\def\1{\bm{1}}

\def\vw{{\bm{w}}}

\DeclareMathAlphabet{\mathsfit}{\encodingdefault}{\sfdefault}{m}{sl}
\SetMathAlphabet{\mathsfit}{bold}{\encodingdefault}{\sfdefault}{bx}{n}

\usepackage[utf8]{inputenc} 
\usepackage[T1]{fontenc}    
\usepackage{hyperref}       
\usepackage{url}            
\usepackage{booktabs}       
\usepackage{amsfonts}       
\usepackage{nicefrac}       
\usepackage{microtype}      

\usepackage{natbib}
\usepackage{amsmath}
\usepackage{amsthm}
\usepackage{graphicx}
\usepackage{subcaption} 

\newtheorem{theorem}{Theorem}[section]

\title{On Regularization of Gradient Descent,\\Layer Imbalance and Flat Minima}

\author{Boris Ginsburg\\
  NVIDIA, Santa Clara, CA USA\\
  \texttt{bginsburg@nvidia.com} \\
}

\begin{document}

\maketitle

\begin{abstract}
  
We analyze the training dynamics for deep linear networks using a new metric --  layer imbalance -- which defines the flatness of a solution. We demonstrate that different regularization methods, such as weight decay or noise data augmentation, behave in a similar way. Training has two distinct phases: 1) `optimization' and 2) `regularization'. First, during the optimization phase, the loss function monotonically decreases, and the trajectory goes toward a minima manifold. Then, during the regularization phase, the layer imbalance decreases, and the trajectory goes along the minima manifold toward a flat area.  Finally, we extend the analysis for stochastic gradient descent and show that SGD works similarly to noise regularization.
\end{abstract}

\section{Introduction}
 
In this paper, we analyze regularization methods used for  training of deep neural networks. To understand how regularization like weight decay and noise data augmentation work, we study gradient descent (GD) dynamics for deep linear networks (DLNs). We study deep networks with \textit{scalar} layers to exclude factors related to over-parameterization and to focus on factors specific to deep models. 
Our analysis is based on the concept of \textit{flat  minima}
\cite{hochreiter1994a}. We call a region in weight space \textit{flat}, if each solution from that region has a similar small loss. 
We show that minima flatness is related to a new metric, \textit{layer imbalance}, which measures the difference between the norm of network layers. Next, we  analyze layer imbalance dynamics of gradient descent (GD) for DLNs using a \textit{trajectory-based} approach \cite{saxe2013}.

With these tools, we prove the following results:
\begin{enumerate}
  \item Standard regularization methods such as weight decay and noise data augmentation, decrease  layer imbalance during training and drive trajectory toward flat minima.
  \item Training for GD with regularization has two  distinct phases: (1) `optimization' and (2) `regularization'. During the optimization phase, the loss monotonically decreases, and the trajectory goes toward minima manifold. During the regularization phase,  layer imbalance  decreases and the trajectory goes along minima manifold toward flat area.
  \item  Stochastic Gradient Descent (SGD) works similarly to  implicit noise regularization.
\end{enumerate}

\section{Linear neural networks}

We begin with a linear regression $y = w \cdot x + b$ with mean squared error on scalar samples $\{x_i,y_i\}$:
\begin{equation}
\label{regression}
 E(w,b)=\frac{1}{N} \sum (w\cdot x_i +b-y_i)^2\rightarrow min
\end{equation}
Let's  center and normalize the training dataset in the following way:
\begin{align} 
\label{norm_data}
  &\sum x_i = 0; \; \; \frac{1}{N}\sum x^2_i = 1;
  &\sum y_i = 0; \; \;  \frac{1}{N}\sum x_i y_i=1. 
\end{align}
The solution for this normalized linear regression is $(w,b)=(1,0)$. 

Next, let's replace  $y=w\cdot x+b$ with a  linear network with $d$ scalar layers $\vw=(w_1,\dots, w_d)$:
\begin{align}
\label{dln}
y=w_1\cdots w_d \cdot x + b
\end{align}
Denote  $\mathbf{W:=w_1\cdots w_d}$. 
The loss function for the new problem is:
\begin{equation*}
   E(\vw, b)=\frac{1}{N} \sum(W \cdot x_i + b - y_i)^2 \rightarrow min
\end{equation*}
Now the loss $E(\vw,.) $ is a non-linear (and non-convex) function with respect to the weights $\vw$. 
For the normalized dataset (\ref{norm_data}), network training is equivalent to the following  problem:
\begin{equation}
\label{eq:dln}
    L(\vw)=(w_1\cdots w_d -1)^2\rightarrow \min
\end{equation}
Such linear networks with depth-2 have been studied in \citet{baldi1989b}, who showed that all minima for the problem (\ref{eq:dln}) are global and that all other critical points are saddles. 

\subsection{Flat minima}
Following Hochreiter  et al \cite{hochreiter1994a}, we are interested in \textit{flat minima} -- ``a region in weight space with the property that each weight from that region has similar small error". In contrast, sharp minima are  regions where the function can increase rapidly.  Let's compute the loss gradient $\nabla L(\vw)$:
\begin{align} 
\label{eq:grad}
  \frac{\partial L}{\partial w_i}
  &=2(w_1\cdots w_d -1)(w_1\cdots w_{i-1}w_{i+1}\cdots w_d)=2(W-1)(W/w_i)
\end{align}
Here we denote $\mathbf{W/w_i:=w_1\cdots w_{i-1}\cdot w_{i+1}\cdots w_d}$ for brevity. 
The minima of loss $L$ are located on hyperbola $w_1\cdots w_d=1$ (Fig.~\ref{fig:dln-loss}). 
Our interest in flat minima is related to training robustness.
Training in the flat area is more stable than in the sharp area: the gradient $\dfrac{\partial L}{\partial w_i}$ vanishes if $|w_i|$ is very large, and the gradient explodes if $|w_i|$ is very small.

\begin{figure}[htb!]
\centering
\includegraphics[width=0.32 \textwidth]{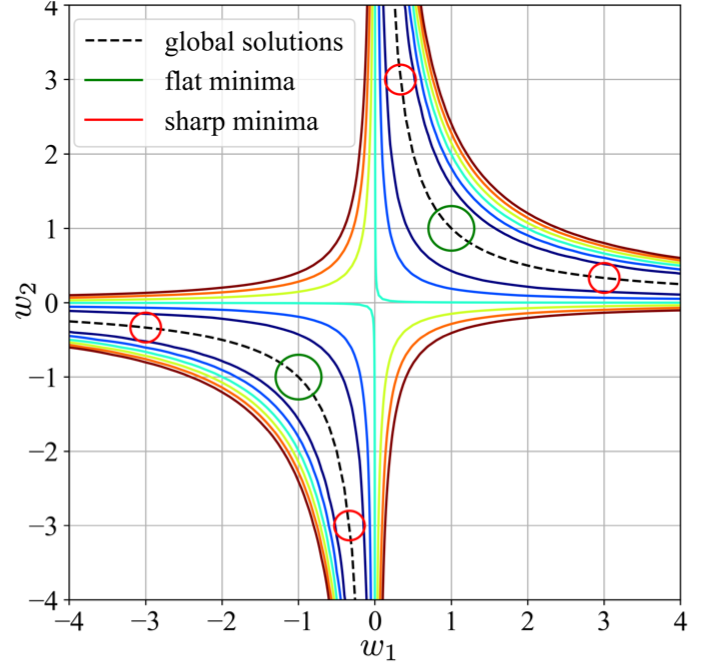}
\caption{2D-contour plot of the loss  $L(w_1, w_2)=(w_1 w_2 - 1)^2$ for the linear network with two layers. The loss $L$ has only global minima, located on the hyperbola $w_1 w_2=1$. Minima  near $(-1,-1)$ and $(1,1)$ are flat, and minima near the axes are sharp.}
\label{fig:dln-loss}
\end{figure}

It was suggested by Hochreiter et al \cite{hochreiter1994b} that flat minima  have smaller generalization errors than sharp minima.  \citet{keskar2016} observed that large-batch training tends to converge towards a sharp minima with a  significant number of large positive eigenvalues of Hessian. They suggested that sharp minima generalize worse than flat minima, which have smaller eigenvalues. In contrast, \citet{Dinh2017} argued that flatness of minima can't be directly applied to explain generalization; since  both flat and sharp minima represent the same function,  they perform equally on a validation set. 

The question of how minima flatness is related to good generalization is out of scope of this  paper.

\subsection{Layer imbalance}

In this section we define a new metric related to the flatness of the minimizer -- layer imbalance.

Dinh \cite{Dinh2017} showed that minima flatness  is defined by the largest  eigenvalue of Hessian $H$:
\begin{equation*}
 H(\vw)=2
 \begin{bmatrix}
  \dfrac{W^2}{w_1^2} & \dfrac{(2W-1)W}{w_1w_2} &\dots &  \dfrac{(2W-1)W}{w_1w_d}\\
  \dfrac{(2W-1)W}{w_2w_1} & \dfrac{W^2}{w_2^2} &\dots& \dfrac{(2W-1)W}{w_2w_d}\\ 
   \dots& \dots&\dots&\dots\\
  \dfrac{(2W-1)W}{w_dw_1} & \dfrac{(2W-1)W}{w_dw_2} &\dots& \dfrac{W^2}{w_d^2}
 \end{bmatrix}
\end{equation*}
The eigenvalues of the Hessian $H(\vw)$  are $\{0,\dots,0,\sum{\dfrac{1}{w_i^2}}\}$. 
Minima close to the axes are sharp. Minima close to the origin are flat. Note that flat minima are balanced: $|w_i|\approx 1$ for all layers.

In the spirit of \cite{Arora2019, neyshabur2015}, let's define \textit{layer imbalance} for a deep linear network:
\begin{equation}
    D(\vw):= \max_{i,j}| \  || w_{i}||^2 - ||w_j||^2 \ |
\end{equation}
Minima with low layer imbalance are flat, and minima with high layer imbalance are sharp.

\section{Implicit regularization for gradient descent}

In this section, we explore the training dynamics for continuous and discrete gradient descent.

\subsection{Gradient descent: convergence analysis}
We start with an analysis of training dynamics for \textit{continuous} GD. By taking a time limit for gradient descent:
$ w_i(t+1)=w_i(t)-\lambda \cdot  \nabla L(\vw)$,  we obtain the  following DEs \cite{saxe2013}:
\begin{align} 
\label{cont_GD}
 \frac{dw_i}{dt}&=-\lambda\frac{\partial L}{\partial w_i}=-2\lambda (W-1)(W/w_i)
\end{align}
For continuous GD, the loss function monotonically decreases:
\begin{align*} 
 \frac{dL}{dt} &
   =\sum \big( \frac{\partial L}{\partial w_i}\cdot \frac{dw_i}{dt}\big) 
   = -4\lambda(W-1)^2 W^2 \big(\sum{\frac{1}{w_i^2}}\big)=-4\lambda W^2 \big(\sum{\frac{1}{w_i^2}}\big)\cdot L(t) \leq 0
\end{align*}
The  trajectory for continuous GD is hyperbola:
$w_i^2(t) - w_j^2(t)=$ const (see
Fig.~\ref{fig:traj_hyperbola}) \citep{saxe2013} . The layer imbalance remains constant during training. So if training starts close to the origin, then a final point will also have a small layer imbalance and a minimum will be flat.

Let's turn from continuous  to  regular gradient descent:\footnote{We omit $t$ in the right part for brevity, so $w_i$ means $w_i(t)$.}
\begin{align}
\label{eq:dgd_step}
  w_i(t+1) &= w_i - 2\lambda \frac{\partial L}{\partial w_i}= w_i - 2\lambda(W-1)(W/w_i)
\end{align}
We would like to find conditions, which would guarantee that the loss monotonically decreases.
For any \textit{fixed} learning rate, one can find a point $\vw$,  such that the loss will increase after the GD step.\footnote{For example, consider the network with 2 layers. The loss $L$ after GD step is:
\begin{align*}
& L(t+1) = \big(w_1(t+1) w_2(t+1)-1\big)^2=\big((w_1 - 2\lambda(w_1 w_2 -1)w_2) (w_2 -  2\lambda(w_1 w_2 -1)w_1)-1\big)^2\\
  & =(w_1 w_2 -1)^2  \big(1 - 2\lambda(w_1^2 + w_2^2) + 4\lambda^2(w_1 w_2-1)\big)^2 = L(t) \Big(1 - 2\lambda(w_1^2 + w_2^2) + 4\lambda^2(w_1 w_2-1)\Big)^2
\end{align*}
For any fixed $\lambda$, one can find $(w_1,w_2)$  with $w_1\cdot w_2  \approx 1$ and large enough $(w_1^2+w_2^2)$ to make 
$ |1 - 2\lambda(w_1^2 + w_2^2) + 4\lambda^2(w_1 w_2-1)|\gg 1$, and therefore the loss will increase: $L(t+1)>L(t)$. 
}
But we can define an \textit{adaptive} learning rate $\lambda(\vw)$ which guarantees that the loss decreases. 

\begin{theorem}
\label{thm:discrtete_gd}
Consider discrete GD (Eq.~\ref{eq:dgd_step}). Assume that  $|W-1|<\dfrac{1}{2}$. If we define an adaptive learning rate 
$ \lambda(\vw)=\dfrac{1}{4\sum(1/w_i^2)} $, then the loss  monotonically converges to 0 with a linear rate.  
\end{theorem}

\begin{proof}
Let's estimate the loss change for a gradient descent step: 
\begin{align*}
 & W(t+1) -1=\prod \big(w_i-2\lambda(W-1)W/w_i\big)-1 \\ 
 &= \prod \big(w_i(1-2\lambda(W-1)W/w_i^2)\big) -1= W\cdot \prod\big(1-2\lambda(W-1)W/w_i^2\big)-1\\ 
 &= W\cdot\Big(1-2\lambda(W-1)W\big(\sum_i 1/w_i^2\big) +4\lambda^2(W-1)^2 W^2\big(\sum_{i\neq j} 1/(w_i^2 w_j^2)\big)\\
  &\quad -8\lambda^3(W-1)^3 W^3\big(\sum_{i\neq j\neq k} 1/(w_i^2 w_j^2 w_k^2)\big)+...\Big) - 1\\
 &= (W-1)\cdot\Big(1 - 2\lambda W^2 \big(\sum_i 1/w_i^2\big) +4\lambda^2(W-1)W^3\big(\sum_{i\neq j} 1/(w_i^2 w_j^2)\big)\\
  &\quad -8\lambda^3(W-1)^2 W^4\big(\sum_{i\neq j\neq k} 1/(w_i^2 w_j^2 w_k^2)\big)+...\Big)= (W-1)\cdot \Big(1-\frac{W}{W-1}\cdot S\Big)
\end{align*}
Here  $S=a_1-a_2+a_3- ... +a_d$ is a series with  
$a_k = \big(2\lambda(W-1)W\big)^k \big(\sum_{i\neq j \neq... m}1/(w_i^2 w_j^2...w_{m}^2)\big)$:
\begin{align*}
  &S =  2\lambda(W-1)W\big(\sum_i 1/w_i^2\big) -4\lambda^2(W-1)^2W^2(\sum_{i\neq j} 1/(w_i^2 w_j^2))\\
  &\quad +8\lambda^3(W-1)^3 W^3\big(\sum_{i\neq j\neq k}1/(w_i^2 w_j^2 w_k^2)\big)+\dots
\end{align*}
Consider the factor $k = \big(1-\frac{W}{W-1}\cdot S\big)$. 
To  prove that $|k|<1$, we consider two cases.

\textbf{CASE 1: $\boldsymbol{(W-1)W <0}$.} 
In this case, the series $S$  can be written as:
\begin{align*}
 &S = -\Big( 2\lambda(1-W)W(\sum_i 1/w_i^2) +4\lambda^2(1-W)^2W^2(\sum_{i\neq j} 1/(w_i^2 w_j^2))+\\
 &\quad +8\lambda^3(1-W)^3 W^3(\sum_{i\neq j\neq k}
    1/(w_i^2 w_j^2 w_k^2)) + ... \Big) 
    \geq 2\lambda(W-1)W(\sum_i 1/w_i^2)\frac{1}{1-q}
\end{align*}
where $q$ is:
\begin{align*}
q &=\Big|\dfrac{a_{k+1}}{a_k}\Big|=\left|\dfrac{(2\lambda(W-1)W)^{k+1}
  \big(\sum_{i\neq...\neq{m+1}}
   1/(w_i^2...w_{m+1}^2)\big)}
 {(2\lambda\ (W-1)W)^k
  \big(\sum_{i\neq ...\neq m} 1/(w_i^2...w_{m}^2)\big)}\right|\\
 &\leq 2\lambda|(W-1)W|\dfrac
  {\big(\sum_{i\neq ...\neq m} 1/(w_i^2 ...w_{m}^2)\big)
    \big(\sum 1/w_i^2\big)}
  {\sum_{i\neq ...\neq m} 1/(w_i^2 ...w_{m}^2)} = 2\lambda|(W-1)W|\big(\sum 1/w_i^2\big)
 \leq 
 \frac{3}{8} 
\end{align*}
So on the one hand: 
$ k=1 - \frac{W}{W-1}S \geq 1- \frac{W}{W-1}\cdot 2\lambda(W-1)W(\sum 1/w_i^2)\frac{1}{1-q} \geq  -\frac{4}{5}$.

On the other hand:
$k< 1-\frac{W}{W-1}\cdot 2\lambda(W-1)W(\sum_i 1/w_i^2)= 1-2\lambda W^2 (\sum 1/w_i^2)  <\frac{7}{8}$.

\textbf{CASE 2: $\boldsymbol{(W-1)W>0}$.} 
In the series $S=a_1-a_2+a_3-...$, all terms $a_i$ are now positive. Since  $q=\Big|\dfrac{a_{k+1}}{a_k}\Big|<\dfrac{3}{8}$, we have that  $\dfrac{5}{8} a_1 < a_1-a_2 < S < a_1$. 

On the one hand:
$k = 1-\frac{W}{W-1}S \geq 1-\frac{W}{W-1} a_1
=1-2\lambda(\sum 1/w_i^2) \cdot W^2  >
  - \frac{1}{8}$.

On the other hand:
$k=1-\frac{W}{W-1}S \leq 1- \frac{5}{8}\cdot\frac{W}{W-1}a_1
 = 1- \frac{5}{8} \cdot 2\lambda(\sum 1/w_i^2)\cdot W^2
  <\frac{59}{64}$.

To conclude, in CASE 1 we prove that  $-\frac{4}{5}< k < \frac{7}{8}$ and in CASE 2 that $ -\frac{1}{8}<k <\frac{59}{64}$. 

Since $L(t+1) < L(t)\cdot k^2$, the loss $L$ monotonically converges to 0 with rate $k^2$.
\end{proof}

\subsection{Gradient descent: implicit regularization}

\begin{theorem}
\label{thm:discrtete_gd_imbalance}
Consider discrete GD (Eq.~\ref{eq:dgd_step}). 
Assume that  $|W-1|<\dfrac{1}{2}$. 
If we define an adaptive learning rate 
$ \lambda(\vw)=\dfrac{1}{4\sum(1/w_i^2)} $, then the layer imbalance  monotonically decreases.
\end{theorem}

\begin{proof}
Let's compute the layer imbalance $D_{ij}$ for the layers $i$ and $j$ after one GD step:
\begin{align*} 
 & D_{ij}(t+1)= w_i(t+1)^2 - w_j(t+1)^2=\big(w_i-2\lambda(W-1)W/w_i\big)^2 -\big(w_j-2\lambda(W-1)W/w_j\big)^2\\
 &=(w_i^2-w_j^2)\cdot
    \big(1-4\lambda^2(W-1)^2 W^2 /(w_i w_j)^2 \big) =D_{ij}\cdot \big(1-4\lambda^2(W-1)^2 W^2 /(w_i w_j)^2 \big)    
 \end{align*}
On the one hand, the factor $k=1-4\lambda^2(W-1)^2 W^2/(w_i w_j)^2 \leq 1$. 

On the other hand:
\begin{align*} 
 k&= 1-4\lambda^2(W-1)^2W^2/(w_i w_j)^2 \geq 1-\lambda^2(W-1)^2W^2(1/w_i^2 + 1/w_j^2)^2\\
 &\geq 1-\lambda^2(\sum 1/w_l^2)^2 (W-1)^2W^2 
  \geq 1-\frac{9}{256} =\frac{247}{256}
\end{align*}
So $D_{ij}(t+1) = k \cdot D_{ij}(t)$ and $\frac{247}{256} < k \leq 1$.  This guarantees that the layer imbalance decreases.
\end{proof}

\textit{Note.} We proved only that the layer imbalance $D$ decreases, but not that $D$ converges to $0$. The layer imbalance may stay large, if the loss $L\rightarrow 0$ too fast or if $W\approx 0$, so the factor $k=1-4\lambda^2\cdot L \cdot W^2(1/(w_i w_j))^2\rightarrow 1$.
To make the layer imbalance $D\rightarrow 0$, we should keep the   loss in certain range, e.g. $\frac{1}{4}< |W-1| < \frac{1}{2}$. For this, we could increase the learning rate if the loss becomes too small, and decrease learning rate if loss becomes large.

\section{Explicit regularization}
In this section, we prove that regularization methods, such as weight decay, noise data augmentation, and continuous dropout, decrease the layer imbalance.

\subsection{Training with weight decay}
As before, we consider the gradient descent for linear network $(w_1,\dots, w_d)$ with $d$ layers. 
Let's add the weight decay (WD) term to the loss:
$\bar{L}(\vw)=(w_1 \cdots w_d - 1)^2 + \mu(w_1^2+\dots +w_d^2)$.

The continuous GD with weight decay is described by the following DEs: 
\begin{align} 
\label{cd_weight_decay}
 \frac{dw_i}{dt}=-\lambda \frac{\partial \bar{L}}{\partial w_i}
   =-2\lambda \big((W-1)(W/w_i) +\mu \cdot w_i\big)
\end{align}
Accordingly, the loss dynamics for continuous GD with weight decay is:
\begin{align*} 
& \frac{dL}{dt} = \sum \frac{\partial{L}}{\partial w_i} \cdot \frac{dw_i}{dt} 
 =-4\lambda   \Big((W-1)^2W^2\big(\sum 1/w_i^2\big)
 + \mu\cdot d\cdot (W-1)W\Big)\\
 &= -4\lambda \big(\sum 1/w_i^2\big) W^2  \big(W-1\big)\big(W-(1-\mu \dfrac{ d}{W (\sum 1/w_i^2)})\big)
\end{align*}
The loss decreases when  $k=(W-1)\big(W-(1-\mu \dfrac{ d}{W (\sum 1/w_i^2)})\big)>0$, outside the \textit{weight decay band}: $ 1-\mu \dfrac{ d}{W(\sum 1/w_i^2)} \leq W \leq 1$.
The width of this band is controlled by the weight decay $\mu$. 

We can divide GD training with weight decay into two phases: (1) optimization and (2) regularization. 
During the first phase, the loss decreases until the trajectory gets into the WD-band. During the second phase, the loss $L$ can oscillate, but the trajectory stays inside the WD-band (Fig.~\ref{fig:gd_wd}) and  goes toward a flat minima area. The layer imbalance monotonically decreases: 
\begin{align*} 
 \frac{d(w_i^2 -w_j^2)}{dt}&=-4\lambda \cdot \Big( \big((W-1)W +\mu w_i^2\big) - \big((W-1)W+\mu w_j^2\big)\Big)= -4\lambda \cdot  \mu \cdot (w_i^2 - w_j^2)
\end{align*}

\begin{figure}[htb!]
  \centering
  \begin{subfigure}[b]{0.32\textwidth}
\centering
\includegraphics[width=0.9\textwidth]{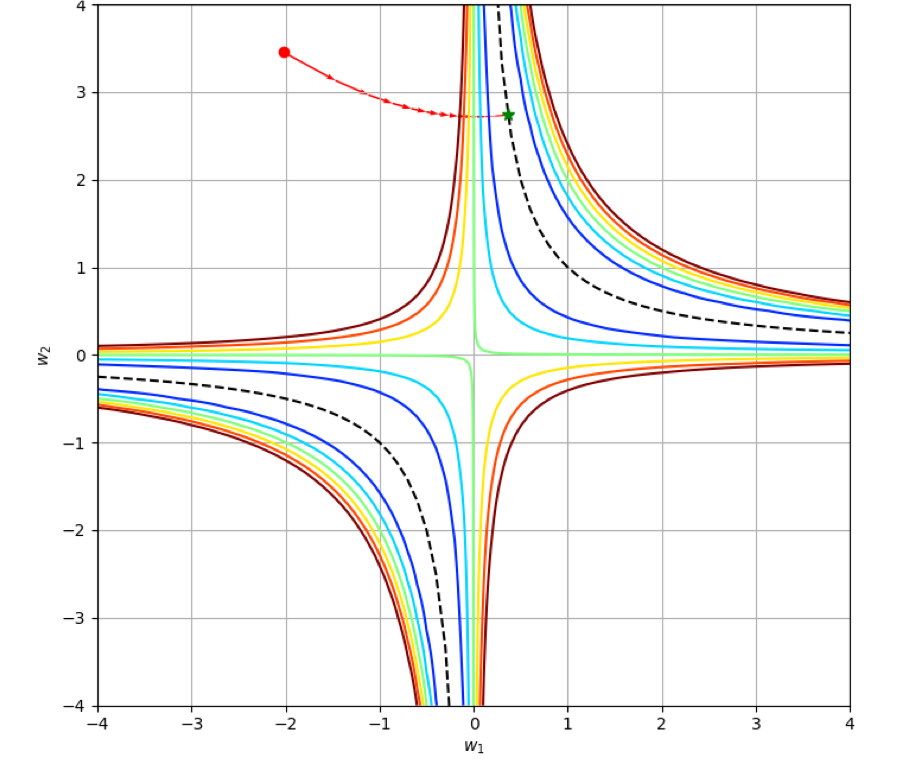}
\caption{Continuous GD }

\label{fig:traj_hyperbola}
   \end{subfigure}%
  \begin{subfigure}[b]{0.32\textwidth}
    \centering
    \includegraphics[width=0.9\textwidth]{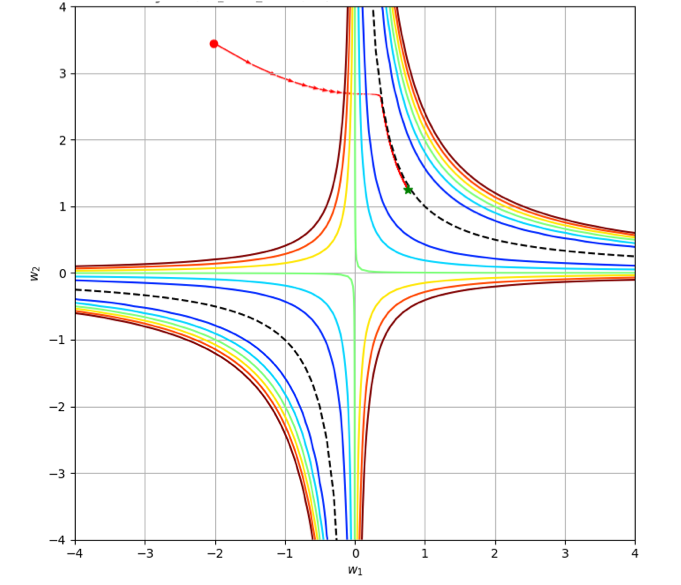}
    \caption{GD with weight decay}
    \label{fig:gd_wd}
   \end{subfigure}%
  \begin{subfigure}[b]{0.30\textwidth}
    \centering
    \includegraphics[width=0.9\textwidth]{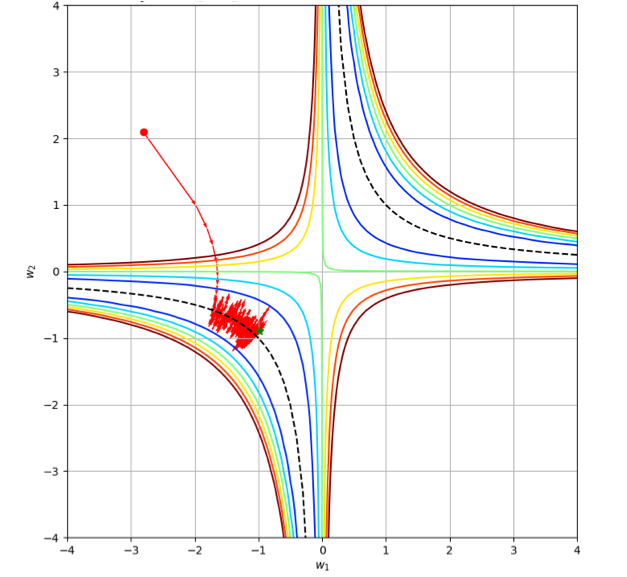}
    \caption{GD with noise augmentation}
    \label{fig:gd_noise}
   \end{subfigure}
  \caption{The training  trajectories for (a) continuous GD, (b) GD with  weight decay,  and (c) GD with noise augmentation. The trajectory for continuous GD is a hyperbola: $ w_i^2(t) - w_j^2(t)=$ const. The trajectories for GD with weight decay and noise augmentation have two parts: (1) optimization -- the trajectory goes toward the minima manifold, and (2) regularization -- the trajectory goes along minima manifold toward a flat area.}
\end{figure}

\subsection{Training with noise augmentation}

\citet{Bishop1995} showed that for shallow networks, training with noise is equivalent to Tikhonov regularization. We extend this result to DLNs.

Let's augment the training data with noise: $\Tilde{x}=x\cdot(1+\eta)$, where the noise $\eta$ has $0-$mean and is bounded: $|\eta|\leq \delta < \frac{1}{2}$. The DLN with noise augmentation can be written in the following form:
\begin{align}
\label{eq:noise_aug}
 \Tilde{y} = w_1\cdots w_d \cdot(1+\eta) x 
\end{align}
This model  also describes  continuous dropout \citep{Srivastava2014}  when  layer outputs $h_i$ are multiplied with the noise: $\Tilde{h}_{i} = (1+\eta)\cdot h_i$. This model can be also used for  continuous drop-connect \citep{Kingma2015b, wan2013} when the  noise is applied to weights: $\Tilde{w}_i=(1+\eta)\cdot w_i$. 

The GD  with noise augmentation is described by the following stochastic DEs:
\begin{align*} 
 \frac{dw_i}{dt} = -\lambda\frac{\partial \Tilde{L}}{\partial w_i}
   = -2\lambda \cdot (1+\eta)(W(1+\eta)-1)(W/w_i)
\end{align*}
Let's consider  loss dynamics:
\begin{align*} 
 &\frac{d L}{dt}= \sum \Big(\frac{\partial L}{\partial w_i} \cdot  \frac{dw_i}{dt}\Big) =
  -4\lambda(1+\eta) W^2 \big(\sum 1/w_i^2 \big) (W-1)(W(1+\eta)-1)\\
  &= -4\lambda(1+\eta)^2 W^2\big(\sum 1/w_i^2\big) 
  \cdot \Big((W-1)(W-\frac{1}{1+\eta}) \Big)
\end{align*}
The loss decreases while the factor  $k=(W-1)(W-\dfrac{1}{1+\eta})=(W-1)(W-1-\dfrac{\eta}{1+\eta})>0$, outside of the \textit{noise band}
$ 1 - \dfrac{\delta}{1+\delta} < W < 1 +\dfrac{\delta}{1-\delta}$.
The  training trajectory  is the hyperbola $w_i^2(t)-w_j^2(t)=$ const. 
When the trajectory gets inside the noise band, it oscillates around the minima manifold, but the layer imbalance remains constant for continuous GD.

Consider now discrete GD with noise augmentation: 
\begin{align} 
\label{eq:dgd_noise}
  w_i(t+1) = w_i-2\lambda(1+\eta)(W(1+\eta)-1)(W/w_i)
\end{align}
For discrete GD,  noise augmentation works similarly to  weight decay.  Training  has two phases: (1) optimization and (2) regularization (Fig.~\ref{fig:gd_noise}). During the optimization phase, the loss decreases until the trajectory hits the noise band. Next, the trajectory oscillates inside the noise band, and the layer imbalance  decreases. The noise variance $\sigma^2$ defines the band width, similarly to the weight decay $\mu$.

\begin{theorem}
Consider discrete GD with noise augmentation (Eq.~\ref{eq:dgd_noise}). Assume that the noise $\eta$ has 0-mean and is bounded: $|\eta|<\delta<\dfrac{1}{2}$. If we
define the adaptive learning rate $\lambda(\vw)=\dfrac{1}{2} \Big(\dfrac{2}{3}\Big)^5 \dfrac{1}{\sum 1/w_i^2}$, then the layer imbalance monotonically decreases inside the noise band $|W-1|<\delta$.
\end{theorem}

\begin{proof}
Let's estimate the layer imbalance:
\begin{align*}
 &w_i^2(t+1)-w_j^2(t+1)\\
  &=\big(w_i-2\lambda(1+\eta)(W(1+\eta)-1)W/w_i\big)^2
  -\big(w_j-2\lambda(1+\eta)(W(1+\eta)-1)W/w_j\big)^2\\
 &= (w_i^2 -w_j^2)+
  4\lambda^2(1+\eta)^2(W(1+\eta)-1)^2\big(W^2/w_i^2- W^2/w_j^2\big)\\
&=(w_i^2-w_j^2) \cdot  
\Big(1-4\lambda^2(1+\eta)^4\big(W-\frac{1}{1+\eta}\big)^2 W^2/(w_i w_j)^2 \Big)
\end{align*}
On the one hand, the factor $k=1-4\lambda^2(1+\eta)^4\big(W-\dfrac{1}{1+\eta}\big)^2 W^2/(w_i w_j)^2 \leq 1$.

On the other hand:
\begin{align*}
 &k=1-4\lambda^2(1+\eta)^4\big(W-\frac{1}{1+\eta}\big)^2 W^2/(w_i w_j)^2\\
 &\geq 1- \lambda^2(1+\eta)^4\big(W-\frac{1}{1+\eta}\big)^2 W^2(1/w_i^2 + 1/w_j^2)^2\\
 & \geq 1-\lambda^2 (1+\eta)^4\big(W-1+\frac{\eta}{1+\eta}\big)^2 W^2 \big(\sum_i 1/w_i^2 \big)^2 \\
 & \geq 1- \lambda^2\big(\sum_i 1/w_i^2 \big)^2 \cdot (1+\delta)^4\big(\delta +\frac{\delta}{1-\delta}\big)^2 (1+\delta)^2 
 \geq 1- \lambda^2\big(\sum_i 1/w_i^2 \big)^2  (3/2)^{10} 
\end{align*}
Taking $\lambda =\dfrac{1}{2} \Big(\dfrac{2}{3}\Big)^5  \dfrac{1}{\sum 1/w_i^2}$ makes $0<k\leq 1$, which proves that the layer imbalance decreases.
\end{proof}
\textit{Note.} We can prove that the layer imbalance $E[D] \rightarrow 0$ if we also assume that all layers are uniformly bounded $|w_i|< C$. This  implies that there is  $\epsilon>0$ such that for all $\vw$ the adaptive learning rate $\lambda(\vw) > \epsilon$, and we can prove that the expectation $E(k)<1$:
\begin{align*}
 & E(k)=
  1- E\Big[4\lambda^2(1+\eta)^4\big(W-\frac{1}{1+\eta}\big)^2 W^2/(w_i w_j)^2 \Big]\\
 &\leq 1-4\lambda^2 W^2/(w_i w_j)^2 \cdot 
   (1+\sigma^2)^2 \frac{\sigma^2}{1+\sigma^2}  \leq 1- 4\lambda^2 \frac{1}{4C^4}\big(1+\sigma^2\big)\sigma^2
   \leq 1-\frac{\lambda^2 \sigma^2}{C^4} 
\end{align*}
This proves that the layer imbalance $D\rightarrow 0$ with rate  
$\big(1-\dfrac{\lambda^2\sigma^2}{C^4}\big)$. 

\section{SGD noise as implicit regularization}
In this section, we show that SGD works as implicit noise regularization, and that the layer imbalance converges to $0$.
As before, we train a DLN $y = W x $  with loss $L(\vw)=\frac{1}{N}\sum (W x_n  - y_n)^2$ on a normalized dataset with $N$ samples $\{x_n,y_n\}$:
\begin{align*} 
  &\sum x_i = 0; \; \; \frac{1}{N}\sum x^2_i = 1; \; \;
  \sum y_i = 0; \; \;  \frac{1}{N}\sum x_i y_i=1. 
\end{align*}
A stochastic gradient for a batch $\bar{B}$ with $B<N$ samples is:
\begin{align*} 
 &\frac{\partial L_B}{\partial w_i} = \frac{1}{|B|} \sum_{
 \bar{B}}2(W x_n^2 - x_n y_n)W/w_i 
 = 2 \Big(W  (\frac{1}{B} \sum_{\bar{B}}x_n^2) -  (\frac{1}{B}\sum_{\bar{B}}x_n y_n) \Big)W/w_i
\end{align*}
If batch size $B\rightarrow N$, then terms $\sum_{\bar{B}}x_n^2\rightarrow\sum_{N} x_n^2 = 1$ and 
$\sum_{\bar{B}} (x_n y_n)\rightarrow \sum_{\bar{B}}(x_n y_n)=1$.

So we can write the stochastic gradient in the following form:
\begin{align*} 
 & &\frac{\partial L_B}{\partial w_i}= 2 \Big(W (1+\eta_1) - (1+\eta_2)\Big)W/w_i= 
 2 \Big(W-1 + (W \eta_1 - \eta_2) \Big)W/w_i 
\end{align*}
The  factor $(1+\eta_1)$ works as noise data augmentation, and the term $\eta_2$ works as label noise.
Both $\eta_1$ and $\eta_2$ have $0$-mean. When loss is small, we can combine both components into one \textit{SGD noise} term: $\eta= W \eta_1 - \eta_2$. SGD noise $\eta$ has $0$-mean.  We assume that  SGD noise variance depends on  batch size in the following way:
$\sigma^2 \approx (\dfrac{1}{B}-\dfrac{1}{N})$.
The trajectory for continuous SGD is described by the stochastic DEs: 
\begin{align*} 
  \frac{dw_i}{dt}=-\lambda \cdot \frac{\partial L_B}{\partial w_i} 
  = -2 \lambda \Big( W-1+\eta \Big) W/w_i 
\end{align*}
Let's start with loss analysis:
\begin{align*} 
 \frac{dL}{dt}
   =-4\lambda W^2 \big(\sum1/w_i^2\big)\cdot (W-1)(W-1+\eta)
\end{align*}
For continuous SGD, the loss decreases anywhere except in the \textit{SGD noise band}: $(W-1)(W-1 + \eta)<0$. 
The band width depends on $B$: the smaller the batch, the wider the band. The SGD training  consists of two parts. First, the loss decreases until the trajectory hits the SGD-noise band. Then the trajectory oscillates inside the noise band. 
The layer imbalance remains constant for continuous SGD. 

Similarly to the noise augmentation, the layer imbalance decreases for discrete SGD:
\begin{align}
\label{eq:sgd}
  w_i(t+1) = w_i-2\lambda(W-1+\eta)W/w_i
\end{align}

\begin{theorem}
\label{thm:sgd}
Consider discrete SGD (Eq.~\ref{eq:sgd}). Assume that $|W-1|<\delta$, and that SGD noise satisfies $|\eta| \leq \delta< 1$. If we define the adaptive learning rate  $\lambda(\vw) = \dfrac{1}{2\delta(1+\delta) (\sum (1/w_i^2)}$, then the layer imbalance monotonically decreases. 
\end{theorem}
\begin{proof}
Let's estimate the layer imbalance:
\begin{align*}
 &w_i^2(t+1)-w_j^2(t+1)=\big(w_i-2\lambda(W-1+\eta)W/w_i\big)^2 -\big(w_j-2\lambda(W-1+\eta)W/w_j\big)^2\\
 &=(w_i^2-w_j^2) \cdot 
 \Big(1-4\lambda^2(W-1+\eta)^2 W^2 /(w_i w_j)^2 \Big)
\end{align*}
On the one hand, the factor 
$k=1-4\lambda^2(W-1+\eta)^2 W^2 /(w_i w_j)^2 \leq 1$.
On the other hand:
\begin{align*} 
 k&=1-4\lambda^2(W-1+\eta)^2 W^2 /(w_i w_j)^2
  \geq 1-2\lambda^2(W-1+\eta)^2 W^2 \big(1/w_i^2+1/w_j^2\big)^2\\
 &\geq 1-4\lambda^2 W^2 \big(\sum 1/w_i^2\big)^2 \cdot ((W-1)^2 +\eta^2)
 \geq 1- 4 \lambda^2 \big(\sum 1/w_i^2\big)^2 \cdot  \delta^2 (1+\delta)^2
\end{align*}
Setting  $\lambda = \dfrac{1}{2\delta(1+\delta)(\sum (1/w_i^2)}$ makes $0<k \leq 1$, which completes the proof.
\end{proof}

The layer imbalance $D\rightarrow 0$ at a rate proportional to the variance of SGD noise.
It was observed by \citet{keskar2016} that SGD training with a large batch leads to sharp solutions, which generalize worse than solutions obtained with a smaller batch.  
This fact directly follows from Theorem~\ref{thm:sgd}. The layer imbalance decreases at a rate $O(1- k\lambda^2 \sigma^2)$.
When a batch size increases, $B\rightarrow N$,  the variance of SGD-noise decreases as ${\approx(\dfrac{1}{B}-\dfrac{1}{N})}$. One can  compensate for smaller SGD noise with additional generalization:  data augmentation,  weight decay, dropout, etc.

\section{Discussion}

In this work, we explore dynamics for gradient descent training of deep linear networks. Using the layer imbalance metric,  we analyze how   regularization methods such as $L_2$-regularization, noise data augmentation, dropout, etc, affect  training dynamics.
We show that for all these methods the training has two distinct phases: optimization and regularization. During the optimization phase, the training trajectory goes from an initial point toward minima manifold, and   loss monotonically decreases. During  the regularization phase, the  trajectory goes along minima manifold toward flat minima, and the layer imbalance monotonically decreases.  We derive an analytical proof that noise augmentation and continuous dropout work similarly to $L_2$-regularization. 
Finally, we show that SGD behaves in the same way as gradient descent with noise regularization.

This work provides an analysis of regularization for scalar  linear networks. We leave the question of how regularization works for over-parameterized nonlinear networks for future research.  The work also gives a few interesting insights into training dynamics, which can lead to new algorithms for large batch training, new learning rate policies, etc. 



\subsubsection*{Acknowledgments}
We would like to thank Vitaly Lavrukhin, Nadav Cohen and Daniel Soudry for the valuable feedback.

\bibliographystyle{plainnat}
\bibliography{bibliography}

\begin{thebibliography}{12}
\providecommand{\natexlab}[1]{#1}
\providecommand{\url}[1]{\texttt{#1}}
\expandafter\ifx\csname urlstyle\endcsname\relax
  \providecommand{\doi}[1]{doi: #1}\else
  \providecommand{\doi}{doi: \begingroup \urlstyle{rm}\Url}\fi

\bibitem[Arora et~al.(2019)Arora, Golowich, Cohen, and Hu]{Arora2019}
S.~Arora, N.~Golowich, N.~Cohen, and W.~Hu.
\newblock A convergence analysis of gradient descent for deep linear neural
  networks.
\newblock In \emph{ICLR}, 2019.

\bibitem[Baldi and Hornik(1989)]{baldi1989b}
P.~Baldi and K.~Hornik.
\newblock Neural networks and principal component analysis: Learning from
  examples without local minima.
\newblock In \emph{Neural Networks 2.1}, page 53–58, 1989.

\bibitem[Bishop(1995)]{Bishop1995}
C.~M. Bishop.
\newblock Training with noise is equivalent to {Tikhonov} regularization.
\newblock \emph{Neural Computation}, 7:\penalty0 108–116., 1995.

\bibitem[Dinh et~al.(2017)Dinh, Pascanu, Bengio, and Bengio]{Dinh2017}
Laurent Dinh, Razvan Pascanu, Samy Bengio, and Yoshua Bengio.
\newblock Sharp minima can generalize for deep nets.
\newblock In \emph{ICML}, 2017.

\bibitem[Hochreiter and Schmidhuber(1994{\natexlab{a}})]{hochreiter1994a}
S.~Hochreiter and J.~Schmidhuber.
\newblock Simplifying neural nets by discovering flat minima.
\newblock In \emph{NIPS}, 1994{\natexlab{a}}.

\bibitem[Hochreiter and Schmidhuber(1994{\natexlab{b}})]{hochreiter1994b}
S.~Hochreiter and J.~Schmidhuber.
\newblock Flat minima search for discovering simple nets, technical report
  fki-200-94.
\newblock Technical report, Fakultat fur Informatik, H2, Technische Universitat
  Munchen, 1994{\natexlab{b}}.

\bibitem[Keskar et~al.(2016)Keskar, Mudigere, Nocedal, Smelyanskiy, and
  Tang]{keskar2016}
N.~S. Keskar, D.~Mudigere, J.~Nocedal, M.~Smelyanskiy, and P.~T.~P. Tang.
\newblock On large-batch training for deep learning: generalization gap and
  sharp minima.
\newblock In \emph{ICLR}, 2016.

\bibitem[Kingma et~al.(2015)Kingma, Salimans, and Welling]{Kingma2015b}
D.~Kingma, T.~Salimans, and M.~Welling.
\newblock Variational dropout and the local reparameterization trick.
\newblock In \emph{NIPS}, 2015.

\bibitem[Neyshabur et~al.(2015)Neyshabur, Salakhutdinov, and
  Srebro]{neyshabur2015}
Behnam Neyshabur, Ruslan Salakhutdinov, and Nathan Srebro.
\newblock Path-sgd: Path-normalized optimization in deep neural networks.
\newblock In \emph{NIPS}, page 2422–2430, 2015.

\bibitem[Saxe et~al.(2013)Saxe, McClelland, and Ganguli]{saxe2013}
Andrew~M. Saxe, James~L. McClelland, and Surya Ganguli.
\newblock Exact solutions to the nonlinear dynamics of learning in deep linear
  neural network.
\newblock In \emph{ICLR}, 2013.

\bibitem[Srivastava et~al.(2014)Srivastava, Hinton, Krizhevsky, Sutskever, and
  Salakhutdinov]{Srivastava2014}
N.~Srivastava, G.~E. Hinton, A.~Krizhevsky, I.~Sutskever, and R.~Salakhutdinov.
\newblock Dropout: a simple way to prevent neural networks from overfitting.
\newblock \emph{Journal of Machine Learning Research}, 2014.

\bibitem[Wan et~al.(2013)Wan, Zeiler, Zhang, Cun, and Fergus]{wan2013}
Li~Wan, Matthew Zeiler, Sixin Zhang, Yann~Le Cun, and Rob Fergus.
\newblock Regularization of neural networks using dropconnect.
\newblock In \emph{ICML}, 2013.

\end{thebibliography}

\end{document}